\documentclass{amsart}
\usepackage[utf8]{inputenc}
\usepackage{pgfplots}
\usepackage{thmtools}
\usepackage{thm-restate}
\usepackage{csquotes}
\usepackage{verbatim}
\usepackage{tikz}
\usepackage[T1]{fontenc}
\usepackage[hidelinks]{hyperref}
\usepackage{url}
\usepackage{booktabs}	
\usepackage{amsfonts}
\usepackage{microtype}
\usepackage{mathtools}
\usepackage{amsthm, amssymb}
\usepackage{subcaption}
\usetikzlibrary{positioning}

\newcommand{\Bernoulli}{\operatorname{Bern}}
\newcommand\G{\mathbb{G}}
\newcommand{\W}{\mathcal{W}^1}
\newcommand\simiid{\stackrel{\text{iid}}{\sim}}

\newcommand\1{1}
\newcommand\dd{\mathrm d}

\newcommand\Unif{\operatorname{Unif}}
\newcommand\R{\mathbb{R}}

\def\P{{\mathbb P}}
\def\E{{\mathbb E}}
\def\N{{\mathbb N}}

\DeclarePairedDelimiter{\abs}{\lvert}{\rvert}
\DeclarePairedDelimiter{\bra}{[}{]}

\newtheorem{definition}{Definition}
\newtheorem{proposition}{Proposition}
\newtheorem{lemma}{Lemma}
\newtheorem{corollary}{Corollary}

\subjclass[2010]{68T05, 05C80, 62G99}

\keywords{graphon, graph kernel, graph spectrum, connectome}
\hypersetup{pdftitle={Classification on Large Networks}, 
pdfauthor={A. Haupt, M. Khatami, T. Schultz, N. Tran}}

\title[Classification on Networks via Motifs]{Classification on Large Networks: A Quantitative Bound via Motifs and Graphons}
\author{Andreas Haupt}
\address{{\normalfont(Andreas Haupt)} Hausdorff Center for Mathematics, University of Bonn, Bonn 53115, Germany}
\email{s6anhaup@uni-bonn.de}
\author{Mohammad Khatami}
\address{{\normalfont(Mohammad Khatami)} Department of Computer Science, University of Bonn, Bonn 53113, Germany}
\email{khatami@cs.uni-bonn.de}
\author{Thomas Schultz}
\address{{\normalfont(Thomas Schultz)} Department of Computer Science, University of Bonn, Bonn 53113, Germany}
\email{schultz@cs.uni-bonn.de}
\author{Ngoc Mai Tran}
\address{{\normalfont(Ngoc Mai Tran)} Department of Mathematics, University of Texas at Austin, Texas TX 78712, USA}
\email{ntran@math.utexas.edu}
\begin{document}
\begin{abstract}
  When each data point is a large graph, graph statistics such as densities of certain subgraphs (motifs) can be used as feature vectors for machine learning. While intuitive, motif counts are expensive to compute and difficult to work with theoretically. Via graphon theory, we give an explicit quantitative bound for the ability of motif homomorphisms to distinguish large networks under both generative and sampling noise. Furthermore, we give similar bounds for the graph spectrum and connect it to homomorphism densities of cycles. This results in an easily computable classifier on graph data with theoretical performance guarantee. Our method yields competitive results on classification tasks for the autoimmune disease \emph{Lupus Erythematosus}. 
\end{abstract}
\maketitle
\section{Introduction}

This paper concerns classification problems when each data point is a large network. In neuroscience, for instance, the brain can be represented by a structural connectome or a functional connectome, both are large graphs that model connections between brain regions. In ecology, an ecosystem is represented as a species interaction network. On these data, one may want to classify diseased vs healthy brains, or a species network before and after an environmental shock. Existing approaches for graph classification can be divided broadly into three groups: (1) use of graph parameters such as edge density, degree distribution, or densities of motifs as features, (2) parametric models such as the stochastic $k$-block model, \cite{abbe2017community}, and (3) graph kernels, \cite{gartner2003graph}, and graph embeddings, \cite{riesen2010graph}. Amongst these methods, motif counting is perhaps the least rigorously studied. Though intuitive, only small motifs are feasible to compute, and thus motif counting is often seen as an ad-hoc method with no quantitative performance guarantee. 

\subsection{Contributions}

In this paper, we formalize the use of motifs to distinguish graphs using graphon theory, and give a tight, explicit quantitative bound for its performance in classification (cf. Theorem \ref{thm:main1}). Furthermore, we use well-known results from graph theory to relate the spectrum (eigenvalues) of the adjacency matrix one-to-one to cycle homomorphism densities, and give an analogous quantitative bound in terms of the spectrum (cf. Theorem \ref{thm:main2}). These results put motif counting on a firm theory, and justify the use of spectral graph kernels for counting a family of motifs. We apply our method to detect the autoimmune disease \emph{Lupus Erythematosus} from diffusion tensor imaging (DTI) data, and obtain competitive results to previous approaches (cf. \S\ref{sec:application}). 

Another contribution of our paper is the first study of a general model for random weighted graphs, \emph{decorated graphons}, in a machine learning context. The proof technique can be seen as a broad tool for tackling questions on generalisations of graphons. There are three key ingredients. The first is a generalization of the Counting Lemma, \cite[Theorem 10.24]{lovasz2012large}, on graphons to decorated graphons. It allows one to lower bound the cut metric by homomorphism densities of motifs, a key connection between motifs and graph limits. The second is Kantorovich duality, \cite[Theorem 5.10]{villani2008optimal}, which relates optimal coupling between measures and optimal transport over a class of functions and which is used in relating spectra to homomorphism densities. In this, Duality translates our problem to questions on function approximation, to which we use tools from approximation theory to obtain tight bounds. Finally, we use tools from concentration of measure to deal with sampling error an generalise known sample concentration bounds for graphons, \cite[Lemma 4.4]{borgs2008convergent}.

Our method extends results for discrete edge weights to the continuous edge weight case. Graphs with continuous edge weights naturally arise in applications such as neuroscience, as demonstrated in our dataset. The current literature for methods on such graphs is limited, \cite{feragen2013scalable, neumann2012efficient}, as many graph algorithms rely on discrete labels, \cite{shervashidze2011weisfeiler, borgwardt2005shortest}. 

\subsection{Related Literature}
Graphons, an abbreviation of the words \enquote{graph} and \enquote{function}, are limits of large vertex exchangeable graphs under the cut metric. For this reason, graphons and their generalizations are often used to model real-world networks, \cite{borgs2014p, veitch2015class, cai2016edge}. Originally appeared in the literature on exchangeable random arrays, \cite{aldous1981representations}, it was later rediscovered in graph limit theory and statistical physics, \cite{lovasz2012large, diaconis2007graph}. 

There is an extensive literature on the inference of graphons from one observation, i.e. \emph{one} large but finite graph, \cite{klopp2017oracle, wolfe2013nonparametric, borgs2015private, airoldi2013stochastic}. This is distinct from our classification setup, where one observes multiple graphs drawn from several graphons. In our setting, the graphs might be of different sizes, and crucially, they are unlabelled: There is no \emph{a priori} matching of the graph nodes. That is, if we think of the underlying graphon as an infinitely large random graph, then the graphs in our i.i.d sample could be glimpses into entirely different neighborhoods of this graphon, and they are further corrupted by noise. A na\"{i}ve approach would be to estimate one graphon for each graph, and either average over the graphs or over the graphons obtained. Unfortunately, our graphs and graphons are only defined up to relabelings of the nodes, and producing the optimal labels between a pair of graphs is NP-complete (via subgraph isomorphism). Thus, inference in our setting is not a mere \enquote{large sample} version of the graphon estimation problem, but an entirely different challenge. 

A method closer to our setup is graph kernels for support-vector machines, \cite{gartner2003graph,vishwanathan2006fast}. The idea is to embed graphs in a high-dimensional Hilbert space, and compute their inner products via a kernel function. This approach has successfully been used for graph classification, \cite{vishwanathan2010graph}. Most kernels used are transformations of homomorphism densities/motifs as feature vectors for a class of graphs (cf \cite[subsection 2.5]{yanardag2015deep}): \cite{shervashidze2009efficient} propose so-called \emph{graphlet counts}\index{graph kernel!graphlet} as features. These can be interpreted as using \emph{induced homomorphism densities} (cf \cite[(5.19)]{lovasz2012large}) as features which can be linearly related to homomorphism densities as is shown in \cite[(5.19)]{lovasz2012large}. The random walk kernel from \cite[p. 135 center]{gartner2003graph} uses the homomorphism densities of all paths as features. Finally, \cite[Prop. 5 and discussion thereafter]{ramon2003expressivity} uses homomorphism densities of trees of height $\le k$ as features. 

 However, as there are many motifs, this approach has the same problem as plain motif counting: In theory, performance bounds are difficult, in practice, one may need to make \emph{ad hoc} choices. Due to the computational  cost, \cite{gartner2003graph}, in practice, only small motifs of size up to 5 have been used for classification, \cite{shervashidze2009efficient}. Other approaches chose a specific class of subgraphs such as paths, \cite{gartner2003graph} or trees, \cite{shervashidze2011weisfeiler}, for which homomorphism densities or linear combinations of them can be computed efficiently. In this light, our Theorem \ref{thm:main2} is a theoretical advocation for cycles, which can be computed efficiently via the graph spectrum. 

\subsection{Organization}
We recall the essentials of graphon theory in \S\ref{sec:background}. For an extensive reference, see \cite{lovasz2012large}. Main results are in \S\ref{sec:main}, followed by applications in \S\ref{sec:application}. Our proofs can be found in the appendix. 

\section{Background}\label{sec:background}
A graph $G = (V,E)$ is a set of vertices $V$ and set of pairs of vertices, called edges $E$. A label on a graph is a one-to-one embedding of its vertices onto $\mathbb{N}$. Say that a random labelled graph is vertex exchangeable if its distribution is invariant under relabelings. 

A labelled graphon $W$ is a symmetric function from $[0,1]^2$ to $[0,1]$. A relabelling $\phi$ is an invertible, measure-preserving transformation on $[0,1]$. An unlabelled graphon is a graphon up to relabeling. For simplicity, we write \enquote{a graphon $W$} to mean \emph{an unlabelled graphon equivalent to the labelled graphon $W$}. Similarly, by \emph{a graph $G$} we mean \emph{an unlabelled graph which, up to vertex permutation, equals to the labelled graph $G$}.

The cut metric between two graphons $W, W'$ is
\begin{equation*}
	\delta_\Box(W,W')=\inf_{\phi,\varphi} \sup_{S,T} \left\lvert \int_{S \times T} W(\varphi(x),\varphi(y)) - W'(\phi(x),\phi(y))\,\, \dd x\,\dd y \right\rvert, 
\end{equation*}
where the infimum is taken over all relabelings $\varphi$ of $W$ and $\phi$ of $W'$, and the supremum is taken over all measurable subsets $S$ and $T$ of $[0,1]$. That is, $\delta_\Box(W,W')$ is the largest discrepancy between the two graphons, taken over the best relabeling possible. A major result of graphon theory is that the space of unlabelled graphons is compact and complete w.r.t. $\delta_\Box$. Furthermore, the limit of any convergent sequence of finite graphs in $\delta_\Box$ is a graphon (\cite[Theorem 11.21]{lovasz2012large}). In this way, graphons are truly \emph{limits of large graphs}.

A motif is an unlabelled graph. A graph homomorphism $\phi \colon F\to G$ is a map from $V(F)$ to $V(G)$ that preserves edge adjacency, that is, if $\{u,v\} \in E(F)$, then $\{\phi(u),\phi(v)\} \in E(G)$. Often in applications, the count of a motif $F$ in $G$ is the number of different embeddings (subgraph isomorphisms) from $F$ to $G$. However, homomorphisms have much nicer theoretical and computational properties (\cite[par. 2.1.2]{lovasz2012large}). Thus, in our paper, \enquote{motif counting} means \enquote{computation of homomorphism densities}. The homomorphism density $t(F,G)$ is the number of homomorphisms from $F$ to $G$, divided by $|V(G)|^{|V(F)|}$, the number of mappings $V(F) \to V(G)$. Homomorphisms extend naturally to graphons through integration with respect to the kernel $W$, \cite[subsec. 7.2.]{lovasz2012large}. That is, for a graph $F$ with $e(F)$ many edges,
\[
t(F,W) = \int_{[0,1]^{e(F)}} \prod_{\{x,y\} \in E(F)} W(x,y)\, \dd x \dd y. 
\]
The homomorphism density for a weighted graph $G$ on $k$ nodes is defined by viewing $G$ as a step-function graphon, with each vertex of $G$ identified with a set on the interval of Lebesgue measure $1/k$. For a graph $G$ and a graphon $W$, write $t(\bullet,G)$ and $t(\bullet,W)$ for the sequence of homormophism densities, defined over all possible finite graphs $F$. 

A finite graph $G$ is uniquely defined by $t(\bullet,G)$. For graphons, homomorphism densities distinguish them as well as the cut metric, that is, $\delta_\Box(W,W') = 0$ iff $t(\bullet,W) = t(\bullet,W')$, \cite[Theorem 11.3]{lovasz2012large}. In other words, if one could compute the homomorphism densities of all motifs, then one can distinguish two convergent sequences of large graphs. Computationally this is not feasible, as $(t(\bullet,W))_{F \text{ finite graph}}$ is an infinite sequence. However, this gives a sufficient condition test for graphon inequality: If $t(F,W) \neq t(F,W')$ for some motif $F$, then one can conclude that $\delta_\Box(W,W') > 0$. We give a quantitative version of this statement in the appendix, which plays an important part in our proof. Theorem \ref{thm:main1} is an extension of this result that accounts for sampling error from estimating $t(F,W)$ through the empirical distribution of graphs sampled from $W$. 

\subsection{Decorated graphons}
Classically, a graphon generates a random unweighted graph $\G(k, W)$ via uniform sampling of the nodes,
\begin{align*}
U_1, \dots, U_k &\simiid \Unif_{[0,1]}& (\G(k, W)_{ij} | U_1, \dots, U_k) &\simiid \Bernoulli(W(U_i,U_j)), \forall i,j \in [k].  
\end{align*}
Here, we extend this framework to decorated graphons, whose samples are random \emph{weighted} graphs. 
\begin{definition}
	Let $\Pi([0,1])$ be the set of probability measures on $[0,1]$. A decorated graphon is a function $\mathcal{W} \colon [0,1]^2 \to \Pi ([0,1])$. 

For $k \in \mathbb{N}$, the $k$-sample of a measure-decorated graphon $\G(k,\mathcal{W})$ is a distribution on unweighted graphs on $k$ nodes, generated by
\begin{align*}
U_1, \dots, U_k& \simiid \Unif_{[0,1]} & (\G(k, \mathcal{W})_{ij} | U_1, \dots, U_k)& \simiid \mathcal{W}(U_i,U_j), \forall i,j \in [k].  
\end{align*}
\end{definition}
We can write every decorated graphon $\mathcal{W}$ as $\mathcal{W}_{W, \mu}$ with $W(x,y)$ being the expectation of $\mathcal{W}(x,y)$, and $\mu (x,y)$ being the centered measure corresponding to $\mathcal{W} (x,y)$. This decomposition will be useful in formulating our main results, Theorems \ref{thm:main1} and \ref{thm:main2}. 

One important example of decorated graphons are \emph{noisy} graphons, that is, graphons perturbed by an error term whose distribution does not vary with the latent parameter: Given a graphon $W: [0,1]^2 \to [0,1]$ and a centered noise measure $\nu \in \Pi([0,1])$, the $\nu$-noisy graphon is the decorated graphon $\mathcal{W}_{W,\mu}$, where $\mu (x,y) = \nu$ is constant, i.e. the same measure for all latent parameters. Hence, in the noisy graphon, there is no dependence of the noise term on the latent parameters. 

 As weighted graphs can be regarded as graphons, one can use the definition of homomorphisms for graphons to define homomorphism numbers of samples from a decorated graphon (which are then random variables). The $k$-sample from a decorated graphon is a distribution on weighted graphs, unlike that from a graphon, which is a distribution on unweighted (binary) graphs. The latter case is a special case of a decorated graphon, where the measure at $(x,y)$ is a centered variable taking values $W(x,y)$ and $1-W(x,y)$. Hence, our theorems generalise results for graphons.

\subsection{Spectra and Wasserstein Distances}

The spectrum $\lambda(G)$ of a weighted graph $G$ is the set of eigenvalues of its adjacency matrix, counting multiplicities. Similarly, the spectrum $\lambda(W)$ of a graphon $W$ is its set of eigenvalues when viewed as a symmetric operator \cite[(7.18)]{lovasz2012large}. It is convenient to view the spectrum $\lambda(G)$ as a counting measure, that is, $\lambda(G) = \sum_\lambda\delta_\lambda$, where the sum runs over all $\lambda$'s in the spectrum. All graphs considered in this paper have edge weights in $[0,1]$. Therefore, the support of its spectrum lies in $[-1,1]$. This space is equipped with the Wasserstein distance (a variant of the earth-movers distance)
\begin{equation}
\label{eq:innerwasserstein}\W (\mu, \nu) = \inf_{\gamma \in \Pi ([-1,1]^2)} \int_{(x,y) \in [-1,1]^2} \lvert x - y \rvert \dd \gamma (x,y)
\end{equation}
for $\mu,\nu \in \Pi([-1,1])$, where the first (second) marginal of $\gamma$ should equal $\mu$ ($\nu$).  Analogously, equip the space of random measures $\Pi(\Pi([-1,1]))$ with the Wasserstein distance
\begin{equation}
\label{eq:outerwasserstein}\W(\bar{\mu}, \bar{\nu}) = \inf_{\gamma \in \Pi (\Pi([-1,1])^2)} \int_{{(\mu,\nu) \in \Pi([-1,1])^2}} \W (\mu,\nu) \dd \gamma (\mu,\nu). 
\end{equation}
where again the first (second) marginal of $\gamma$ should equal $\bar\mu$ ($\bar\nu$).

Equation \eqref{eq:outerwasserstein} says that one must first find an optimal coupling of the eigenvalues for different realisations of the empirical spectrum and then an optimal coupling of the random measures. Equation~\eqref{eq:innerwasserstein} is a commonly used method for comparing point clouds, which is robust against outliers, \cite{mullen2010signing}. Equation \eqref{eq:outerwasserstein} is a natural choice of comparison of measures on a continuous space. Similar definitions have appeared in stability analysis of features for topological data analysis, \cite{bubenik2015statistical}. 

\section{Graphons for Classification: Main Results}\label{sec:main}
Consider a binary classification problem where in each class, each data point is a finite, weighted, unlabelled graph. We assume that in each class, the graphs are i.i.d realizations of some underlying decorated graphon $\mathcal{W} = \mathcal{W}_{W,\mu}$ resp. $\mathcal{W}' = \mathcal{W}_{W',\mu'}$. Theorem \ref{thm:main1} says that if the empirical homomorphism densities are sufficiently different in the two groups, then the underlying graphons $W$ and $W'$ are different in the cut metric. Theorem \ref{thm:main2} gives a similar bound, but replaces the empirical homomorphism densities with the empirical spectra. Note that we allow for the decorated graphons to have different noise distributions and that noise may depend on the latent parameters.

Here is the model in detail. Fix constants $k, n \in \N$. Let $\mathcal{W}_{W,\mu}$ and $\mathcal{W}_{W',\mu'}$ be two decorated graphons. Let 
\begin{align*}
G_1, \dots, G_n&\simiid \G(k, \mathcal{W}_{W,\mu})&
G'_1, \dots, G'_n&\simiid \G(k, \mathcal{W}_{W',\mu'})
\end{align*}
be weighted graphs on $k$ nodes sampled from these graphons. For a motif graph $F$ with $e(F)$ edges, let 
\[
\bar{t}(F) \coloneqq \frac{1}{n}\sum_{i=1}^n \delta_{t(F,G_i)}
\] 
be the empirical measure of the homomorphism densities of $F$ with respect to the data $(G_1, \ldots, G_n)$ and analogously $\bar t'(F)$ the empirical measure of the homomorphism densities of $(G_1', \ldots, G'_n)$.
\begin{restatable}{theorem}{mainone}\label{thm:main1}
There is an absolute constant $c$ such that with probability $1-2\exp\left(\frac{kn^{-\frac{2}{3}}}{2e(F)^2}\right)-2e^{-.09cn^{\frac{2}{3}}}$ and weighted graphs $G_i$, $G_i'$, $i=1, \dots, n$ generated by decorated graphons $\mathcal{W}_{W,\mu}$ and $\mathcal{W}_{W',\mu'}$,
\begin{equation}
	\delta_\Box(W, W') \ge e(F)^{-1} (\W(t, \bar t)- 9n^{-\frac{1}{3}}).\label{eq:bound.by.t.graphon}
\end{equation}
\end{restatable}
Note that the number of edges affect both the distance of the homomorphism densities $\W (\bar t, \bar t')$ and the constant $e(F)^{-1}$ in front, making the effect of $e(F)$ on the right-hand-side of the bound difficult to analyze. Indeed, for any fixed $v \in \N$, one can easily construct graphons where the lower bound in Theorem \ref{thm:main1} is attained for $k, n \to \infty$ by a graph with $v = e(F)$ edges. Note furthermore, that the bound is given in terms of the expectation of the decorated graphon, $W$, unperturbed by variations due to $\mu$ resp. $\mu'$. Therefore, in the large-sample limit, motifs as features characterise exactly the expectation of decorated graphons.

Our next result utilizes , Theorem \ref{thm:main1} and Kantorovich duality to give a bound on $\delta_\Box$ with explicit dependence on $v$. Let $\overline{\lambda}$, $\overline{\lambda}'$ be the empirical random spectra in the decorated graphon model, that is,
$\overline{\lambda} = \frac{1}{n}\sum_{i=1}^n \lambda(G_i)$, $\overline{\lambda}' = \frac{1}{n}\sum_{i=1}^n \lambda(G'_i)$.
\begin{restatable}{theorem}{maintwo}\label{thm:main2}
There is an absolute constant $c$ such that the following holds: Let $v \in \N$. With probability $1-2v\exp\left(\frac{kn^{-\frac{2}{3}}}{2v^2}\right)-2ve^{-.09cn^{\frac{2}{3}}}$, for weighted graphs generated by decorated graphons $\mathcal{W}_{W,\mu}$ and $\mathcal{W}_{W',\mu'}$,
	\begin{equation*}
		\delta_\Box (W, W')\ge v	^{-2}2^{-1} (4e)^{-v} \left(\mathcal{W}_{\W}^1 (\bar \lambda, \bar \lambda')- \frac{3}{\pi v} - 18v(4e)^vn^{-\frac{1}{3}}\right)
	\end{equation*}
	\end{restatable}
The parameter $v$ can be interpreted as the length of the longest cycle that is involved as a motif. Theorems \ref{thm:main1}  and \ref{thm:main2} give a test for graphon equality. Namely, if $\W(\bar \lambda', \bar\lambda)$ is large, then the underlying graphons $W$ and $W'$ of the two groups are far apart. This type of sufficient condition is analogous to the result of \cite[Theorem 5.5]{bubenik2015statistical} from topological data analysis. It should be stressed that this bound is purely nonparametric. In addition, we do not make any regularity assumption on either the graphon or the error distribution $\mu$. The theorem is stable with respect to transformations of the graph: A bound analogous to Theorem \ref{thm:main2} holds for the spectrum of the graph Laplacian and the degree sequence, as we show in the appendix in \S \ref{sec:degreefeatures}. 
In addition, having either $k$ or $n$ fixed is merely for ease of exposition. We give a statement with heterogenous $k$ and $n$ in the appendix in \S \ref{sec:hetsampsizes}.

\section{An Application: Classification of Lupus Erythematosus}\label{sec:application}

\emph{Systemic Lupus Erythematosus} (SLE) is an autoimmune disease of connective tissue. Between 25-70\% of patients with SLE have neuropsychiatric symptoms (NPSLE), \cite{feinglass1976neuropsychiatric}. The relation of neuropsychiatric symptoms to other features of the disease is not completely understood. Machine learning techniques in combination with expert knowledge have successfully been applied in this field, \cite{khatami2015bundlemap}. 
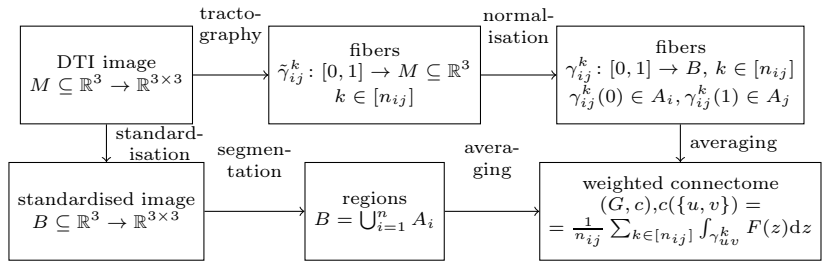
\begin{figure*}[ht]
\centering
\begin{tikzpicture}[node distance=0.5cm and 1cm, snode/.style={draw},every node/.style={font={\scriptsize},align=center},auto, minimum width=0.6cm, minimum height = 1.3cm]
\node[snode] (r){DTI image\\$M \subseteq \R^3 \to \R^{3\times 3}$};
\node[snode] (normalised)[below=of r]{standardised image\\$B \subseteq \R^3 \to \R^{3\times 3}$};
\node[snode] (fiber)[right=of r]{fibers\\$\tilde\gamma_{ij}^k \colon [0,1] \to M\subseteq\R^3$\\$k \in [n_{ij}]$};
\node[snode]  (segmentation)[below =of fiber]{regions\\$ B = \bigcup_{i=1}^n A_i$};
\node[snode] (normalisedfiber)[right=of fiber]{fibers\\$\gamma_{ij}^k \colon [0,1] \to B$, $k \in [n_{ij}]$\\$\gamma_{ij}^k (0) \in A_i, \gamma_{ij}^k (1) \in A_j$};
\node[snode]  (weightedconn)[below=of normalisedfiber]{weighted connectome\\$(G, c)$,$c(\{u, v\}) =$ \\ $= \frac{1}{n_{ij}} \sum_{k \in [n_{ij}]}\int_{\gamma_{uv}^k} F(z) \dd z$};
\path[->] (r) edge node {standard-\\isation} (normalised)
(normalised) edge node{segmen-\\tation} (segmentation)
(r) edge node{tracto-\\graphy} (fiber)
(segmentation) edge node  {avera-\\ging} (weightedconn)
(fiber) edge node {normal-\\isation}  (normalisedfiber)
(normalisedfiber) edge node {averaging}  (weightedconn);
\end{tikzpicture}
\caption{Preprocessing pipeline for weighted structural connectomes. A brain can be seen as a tensor field $B: \R^3 \to \R^{3\times 3}$ of flows. The support of this vector field is partitioned into regions $A_1, \ldots, A_n$, called brain regions. Fibers are parametrized curves from one region to another. Each scalar function $F: \R^3 \to \R$ (such as average diffusivity (AD) and fractional anisotropy (FA)) converts a brain into a weighted graph on $n$ nodes, where the weight between regions $i$ and $j$ is $F$ averaged or integrated over all fibers between these regions.}
\label{fig:pipeline}
\end{figure*}

We analyse a data set consisting of weighted graphs. The data is extracted from diffusion tensor images of 56 individuals, 19 NPSLE, 19 SLE without neuropsychiatric symptoms and 18 human controls (HC) from the study \cite{schmidt2014diminished}. The data was preprocessed to yield 6 weighted graphs on $1106$ nodes for each individual. Each node in the graphs is a brain region of the hierarchical Talairach brain atlas by \cite{talairach1988co}.

The edge weights are various scalar measures commonly used in DTI, averaged or integrated along all fibres from one brain region to another as in the pipeline depicted in Figure \ref{fig:pipeline}. These scalar measures are the \emph{total number} (of fibers between two regions), the \emph{total length} (of all fibers between two regions), \emph{fractional anisotropy} (FA), \emph{mean diffusivity} (MD), \emph{axial diffusivity} (AD) and \emph{radial diffusivity} (RD), cf. \cite{basser2000vivo}.

The paper \cite{khatami2015bundlemap} used the same dataset \cite{schmidt2014diminished}, and considered two classification problems: HC vs NPSLE, and HC vs SLE. Using 20 brain fibers selected from all over the brain (such as the fornix and the anterior thalamic radiation) they used manifold learning to track the values AD, MD, RD and FA along fibers in the brain. Using nested cross-validation, they obtain an optimal disretisation of the bundles, and use average values on parts of the fibers as features for support-vector classification. They obtained an accuracy of $73\%$ for the HC vs. NPSLE and $76\%$ for HC vs. SLE, cf Table \ref{table}.

To directly compare ourselves to \cite{khatami2015bundlemap}, we consider the same classification problems. For each weighted graph we reduce the dimension of graphs by averaging edge weights of edges connecting nodes in the same region on a coarser level of the Talairach brain atlas, \cite{talairach1988co}. Inspired by Theorem \ref{thm:main2}, we compute the spectrum of the adjacency matrix, the graph Laplacian and the degree sequence of the dimension-reduced graphs. We truncate to keep the eigenvalues smallest and largest in absolute value, and plotted the eigenvalue distributions for the six graphs, normalized for comparisons between the groups and graphs (see Figure \ref{fig:density}). We noted that the eigenvalues for graphs corresponding to length and number of fibers show significant differences between HC and NPSLE. Thus, for the task HC vs NPSLE, we used the eigenvalues from these two graphs as features (this gives a total of 40 features), while in the HC vs SLE task, we use all 120 eigenvalues from the six graphs. Using a leave-one-out cross validation with $\ell^1$-penalty and a linear support-vector kernel, we arrive at classification rates of $78\%$ for HC vs. NPSLE and $67.5\%$ for HC vs. SLE both for the graph Laplacian. In a permutation test as proposed in \cite{ojala2010permutation}, we can reject the hypothesis that the results were obtained by pure chance at 10\% accuracy. Table \ref{table} summarises our results. 
 
\begin{figure*}[ht]
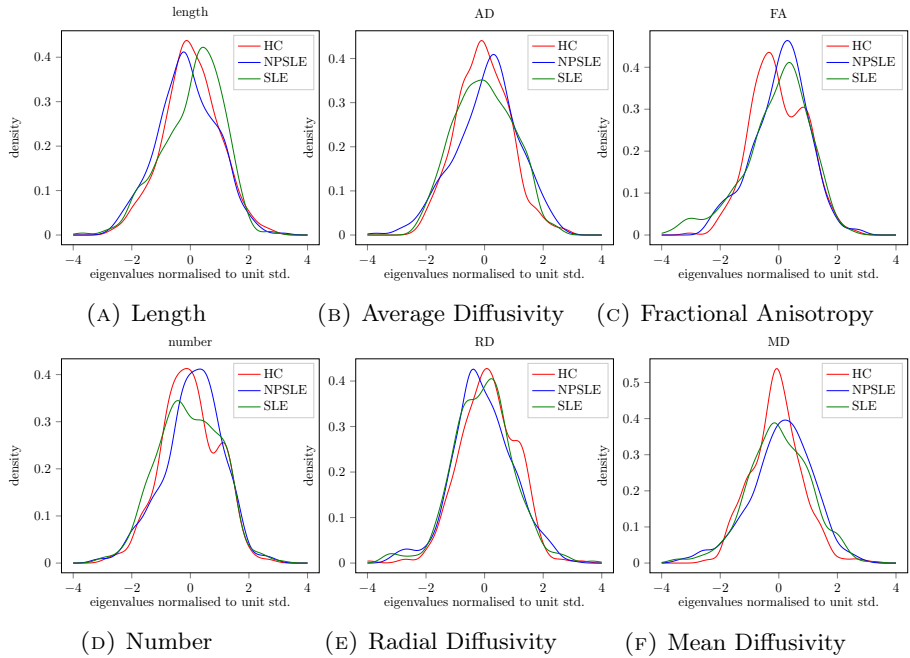

\centering
\begin{subfigure}[t]{.3\linewidth}

\caption{Length}
\end{subfigure}
\begin{subfigure}[t]{.3\linewidth}
%
\caption{Average Diffusivity}
\end{subfigure}
\begin{subfigure}[t]{.3\linewidth}
%
\caption{Fractional Anisotropy}
\end{subfigure}
\begin{subfigure}[t]{.3\linewidth}
%
\caption{Number}
\end{subfigure}
\begin{subfigure}[t]{.3\linewidth}
%
\caption{Radial Diffusivity}
\end{subfigure}
\begin{subfigure}[t]{.3\linewidth}
%
\caption{Mean Diffusivity}
\end{subfigure}
\caption{Density of first and last ten eigenvalues (normalised to zero mean unit standard deviation) of the graph Laplacian for all six values.}
\label{fig:density}
\end{figure*}

\begin{table*}[ht]
\begin{centering}
\begin{tabular}{ccc}
 & HC vs. NPSLE & HC vs. SLE \\
\toprule
\cite{khatami2015bundlemap} & 76\% & 73\% \\
\midrule
Eigenvalues & 78.3\% & 67.5\% \\
\bottomrule
\end{tabular}

\caption{Result comparison. Our spectral method performs comparable to \cite{khatami2015bundlemap}, who used manifold learning and expert knowledge to obtain the feature vectors. Our method is significantly simpler computationally and promises to be a versatile tool for graph classification problems.}\label{table}
\end{centering}
\end{table*}

\section{Conclusion}
In this paper, we provide estimates relating homomorphism densities and distribution of spectra to the cut metric without any assumptions on the graphon's structure. This allows for a non-conclusive test of graphon equality: If homomorphism densities or spectra are sufficiently different, then also the underlying graphons are different. We study the \emph{decorated graphon} model as a general model for random weighted graphs. We show that our graphon estimates also hold in this generalised setting and that known lemmas from graphon theory can be generalised. In a neuroscience application, we show that despite its simplicity, our spectral classifier can yield competitive results. Our work opens up a number of interesting theoretical questions, such as restrictions to the stochastic $k$-block model.
\clearpage
\bibliographystyle{plain}
\bibliography{motifs}
\onecolumn
\section{Proof of Theorem \ref{thm:main1}}
\mainone*
\subsection{Auxiliary Results}
The following result is a generalisation of \cite[Lemma 4.4]{borgs2008convergent} to weighted graph limits.
\begin{lemma}\label{thm:homdensconc}
	Let $\mathcal{W}=\mathcal{W}_{W,\mu}$ be a decorated graphon, $G \sim \G(k, \mathcal{W})$. Let $F$ be an unweighted graph with $v$ nodes. Then 
	 with probability at least $1-2\exp\left(\frac{k\varepsilon^2}{2v^2}\right)$,
	 \begin{equation}
	 \abs{t(F,G) - t(F,W)} < \varepsilon. \label{eq:wantedconcentrationhom}
	 \end{equation}
	\end{lemma}
\begin{proof}
We proceed in three steps. First, give a different formulation of $t(F,W)$ in terms of an expectation. Secondly, we show that this expectation is not too far from the expectation of $t(F,G)$. Finally, we conclude by the method of bounded differences that concentration holds. 
\begin{enumerate}
\item 	Let $t_{\operatorname{inj}}(F,G)$ be the injective homomorphism density, which restricts the homomorphisms from $F$ to $G$ to all those ones that map distinct vertices of $F$ to distinct vertices in $G$, cf. \cite[(5.12)]{lovasz2012large}. Let $G \sim \G(k, \mathcal{W})$ and $X$ be $G$'s adjacency matrix. As a consequence of exchangeability of $X$, it is sufficient in the computation of $t_\text{inj}$ to consider one injection from $V(F)$ to $V(G)$ instead of the average of all such. Without loss, we may assume that $V(F) = [v]$ and $V(G) = [k]$. Hence, for the identity injection $[k] \hookrightarrow [n]$,
\[
	\E[t_\text{inj}(F,X_n)] = \E\left[ \prod_{\{i,j\}\in E(G)} X_{ij}\right].
\]
Let $U_1, \dots, U_n$ be the rows and columns in sampling $X$ from $G$. Then
\begin{align*}
\E\bra*{ \prod_{\{i,j\} \in E(G)} X_{ij}} &= \E\bra*{\E\left[ \prod_{\{i,j\} \in E(G)} X_{ij} \middle| U_1, \dots, U_n\right]}\\
&=\E\bra*{ \prod_{\{i,j\}\in E(G)} (W (U_i, U_j) + \mu(U_i, U_j))}
\end{align*}
We multiply out the last product, and use that $\mu(U_i, U_j)$ are independent and centered to see that all summands but the one involving only terms from the expectation graphon vanish, i.e. 
\[
\E\bra*{ \prod_{\{i,j\}\in E(G)} X_{ij}} = \E\bra*{ \prod_{\{i,j\} \in E(G)} W (U_i, U_j)} = t(F, W)
\]
\item Note that the bound in the theorem is trivial for $\varepsilon^2 \le \ln 2 \frac{2k^2}{n} = 4\ln 2 \frac{k^2}{2n} $. Hence, in particular, $\varepsilon \le 4\ln 2 \frac{k^2}{2n}$.

  Furthermore, $\abs{t(F,X) -t(F,W )}\le \frac{1}{k} \binom{v}{2} + \abs{t(F,X) -\E[t(F,X)] } \le \frac{v^2}{2k}  + \abs{t(F,X) -\E[t(F,X)] }$ by the first part and the bound on the difference of injective homomorphism density and homomorphism density \cite[Lemma 2.1]{lovasz2006limits}.
Hence
\begin{align*}
	\P[\abs{t(F,X_n) -t(F,\E\mathcal{W} )} \ge \varepsilon] &\le \P\left[\abs{t(F,X_n) -\E[t(F,X_n)]}  \ge \varepsilon + \frac{1}{n} \binom{k}{2}\right]\\
	& \le \P\left[\abs{t(F,X_n) -\E[t(F,X_n)]}  \ge \varepsilon\left(1-\frac{1}{4 \ln 2}\right)\right].
\end{align*}
Set $\varepsilon' = \varepsilon\left(1-\frac{1}{4 \ln 2}\right)$.
Let $X$ be the adjacency matrix of $G \sim \G(n,\mathcal{W})$ sampled with latent parameters $U_1, \dots, U_n$. Define a function depending on $n$ vectors where the $i$-th vector consists of all values relevant to the $i$-th column of the array $X_n$, that is $U_i, X_1, \dots, X_n$. In formulas, 
\begin{multline*}
f \colon \bigtimes_{i=1}^n [0,1]^{i+1} \to [0,1], \\(a_1, \dots, a_n) = ((u_1,x_{11}),(u_2,x_{12},x_{22}), \dots, (u_n,x_{1n}, \dots, x_{nn})) \\\mapsto \E[t (F,(X_{ij})_{1 \le i,j \le n})|U_1=u_1, \dots, U_n = u_n, X_{11} = x_{11}, \dots, X_{nn} = x_nn].
\end{multline*}
We note that the random vectors $(U_i, X_{1i}, X_{2i}, \dots, X_{ni})$ are mutually independent for varying $i$. 
Claim: 
\[
\abs{f((a_1, \dots, a_n) - f((b_1, \dots, b_n))} \le \sum_{i=1}^n \frac{k}{n}\1_{a_i \neq b_i}
\]
If this claim is proved, then we have by  {McDiarmid}'s inequality \cite[(1.2) Lemma]{mcdiarmid1989method}, 
\begin{multline*}
	\P[\abs{t(F, X_n) - t(F, \E\mathcal{W})} \ge \varepsilon'] \\\le 2 \exp \left(-\frac{2\varepsilon'^2}{n\left(\frac{k}{n}\right)^2} \right) \le 2\exp \left(-\frac{2\varepsilon'^2n}{k^2} \right)= 2\exp\left(-\frac{2 n\varepsilon'^2}{k^2}\right),
\end{multline*}
Which implies the theorem by basic algebra.  

Let us now prove the claim: It suffices to consider $a,b$ differing in one coordinate, say $n$. By the definition of the homomorphism density of a weighted graph, $t(F,X)$ can be written as
\[
\int g(x_1, \dots, x_k) \dd \Unif_{[n]}^{k}((x_i)_{i \in [k]})
\]
for $g(x_1, \dots, x_k) = \prod_{\{i,k\} \in E(G)} X_{x_ix_k}$. We observe $0 \le g \le 1$ (in the case of graphons, one has $g \in \{0,1\}$). It hence suffices to bound the measure where the integrand $g$ depends on $a_i$ by $\frac{k}{n}$. This is the case only if if $x_\ell =i$ at least for one $\ell \in [k]$. But the probability that this happens is upper bounded by, 
\[
	1-\left(1-\frac{1}{n}\right)^k \le \frac{k}{n},
\]
by the  {Bernoulli} inequality. This proves the claim and hence the theorem. 
\end{enumerate}
\end{proof}
\begin{lemma}[{\cite[Lemma 10.23]{lovasz2012large}}]	\label{lem:cut}
Let $W,W'$ be graphons and $F$ be a motif. Then
\[
\abs{t(F,W) - t(F,W')} \le e(F)\delta_\Box(W,W')
\]
\end{lemma}
\begin{lemma}\label{lem:expbound}
Let $\mu \in \Pi ([0,1])$ and let $\mu_n$ be the empirical measure of $n$ iid samples of $\mu$. Then 
\[
\E[\W (\mu,\mu_n)] \le 3.6462 n^{-\frac{1}{3}}
\]
\end{lemma}
The strategy of prove will be to adapt a proof in \cite[Theorem 1.1]{horowitz1994mean} to the $1$-Wasserstein distance. 
\begin{proof}
	Let $X \sim \mu$, $Y \sim N(0,1)$ and $\mu^\sigma = \operatorname{Law}(X+Y)$. Then for any $\nu \in \Pi ([0,1])$, by results about the standard normal distribution, $W(\nu, \nu^\sigma)\le \E[\lvert Y \rvert] = \sigma \sqrt{\frac{2}{\pi}}$. Hence, by the triangle inequality 
	\[
		\W (\mu, \mu_n) \le 2 \sqrt{\frac{2}{\pi}} \sigma + \W(\mu^\sigma, \mu_n^\sigma).
	\]
	As the discrete norm dominates the absolute value metric on $[0,1]$, $\W(\mu^\sigma, \mu_n^\sigma) \le \| \mu^\sigma - \mu_n^\sigma \|_{\operatorname{TV}}$. Note that $\mu_n^\sigma$ and $\mu^\sigma$ have densities $f^\sigma$, $f_n^\sigma$. This means, as $\| \mu^\sigma -\mu_n^\sigma \|_{\operatorname{TV}} = \int \lvert f^\sigma_n (x) - f^\sigma(x) \rvert \dd x$,
	\[
		\W(\mu^\sigma, \mu_n^\sigma) \le \int \lvert f^\sigma_n (x) - f^\sigma(x) \rvert \dd x \le \sqrt{2\pi}\sqrt{\int (\abs{x}^{2} +1) \lvert f^\sigma_n (x) - f^\sigma(x) \rvert^2 \dd x },
	\]
	where the last inequality is an application of \cite[(2.2)]{horowitz1994mean}. Now observe that by the definitions of $f^\sigma$ and $f_n^\sigma$,  $\E[\lvert f^\sigma_n (x) - f^\sigma(x) \rvert^2 \le n^{-1}\int\phi_\sigma^2(x-y) \dd \mu (y)$, where $\phi_\sigma$ is the standard normal density. Hence
	\[
	\E[\W(\mu_n^\sigma,\mu^\sigma)]\le \sqrt{2\pi} n^{-\frac{1}{2}}\sqrt{\int (\abs{x}^{2} +1) \int\phi_\sigma^2(x-y) \dd \mu (y)  \dd x } 
	\]
	By basic algebra, $\phi_\sigma^2(x)=\frac{1}{2\sigma}\pi^{-\frac{1}{2}}\phi_{\frac{\sigma}{\sqrt{2}}}(x)$. This implies for $Z\sim N(0,1)$ by a change of variables
	\begin{multline*}
	\int (\abs{x}^{2} +1) \int\phi_\sigma^2(x-y)\dd\mu(y)\dd x\\\le \frac{1}{2\sigma\sqrt{\pi}}(1+2(\sigma^2\E[Z^2]+\int\abs{y}^2 \dd \mu (y)))\le\sigma^{-1}2^{-1}\pi^{-\frac{1}{2}}(1+2(\sigma^2x+1))\le\sigma^{-1}\pi^{-\frac{1}{2}}\frac{3}{2}
	\end{multline*}
	Hence $\E[\W(\mu_n^\sigma,\mu^\sigma)]\le\frac{3}{2}\sqrt{2}n^{-\frac{1}{2}}\sigma^{-\frac{1}{2}}=\frac{3}{\sqrt{2}}n^{-\frac{1}{2}}\sigma^{-\frac{1}{2}}$ and
	\[
	\E[\W(\mu_n, \mu)]\le 2\sqrt{\frac{2}{\pi}}\sigma+\frac{3}{\sqrt{2}}n^{-\frac{1}{2}}\sigma^{-\frac{1}{2}}.
	\]
	Choosing $\sigma$ optimally by a first-order condition, one arrives at the lemma. 
\end{proof}
\begin{lemma}[{{\cite[Theorem 2]{2013arXiv1312.2128F}}}]\label{lem:wassersteinconc}
	Let $\mu \in \mathcal{P}(\R)$ such that for $X \sim \mu$, $\ell =\E[e^{\gamma X^\alpha}]<\infty$ for some choice of $\gamma$ and $\alpha$. Then one has with probability at least $1-e^{-cn\varepsilon^2}$
	\[
		\W (\mu_n, \mu) \le \varepsilon
	\]
	for any $\varepsilon \in [0,1]$ and $c$ only depending on $\ell, \gamma$ and $\alpha$. 
\end{lemma}
\subsection{Proof of Theorem \ref{thm:main1}}
\begin{proof}[Proof of Theorem \ref{thm:main1}]
Let $G \sim \G(k, \mathcal{W})$ and $G' \sim \G(k, \mathcal{W}')$. 
By combining Lemmas \ref{lem:expbound} and \ref{lem:wassersteinconc}, we get that with probability at least $1-2e^{-.09cn^{\frac{2}{3}}}$, 
\[
	\W(\bar t, \bar t') \le \W(t(F,G), t(F,G')) + 8n^{-\frac{1}{3}}
\]
In addition, by Lemma \ref{thm:homdensconc}, with probability at least $1-2\exp\left(\frac{kn^{-\frac{2}{3}}}{2v^2}\right)-2e^{-.09cn^{\frac{2}{3}}}$ one also has
\[
	\W(t(F,G), t(F,G')) \le \abs{t(F,W) - t(F,W')} + n^{-\frac{1}{3}}
\]
Upon application of Lemma \ref{lem:cut} and rearranging, one arrives at the theorem. 
\end{proof}
\section{Proof of Theorem \ref{thm:main2}}
\maintwo*
\subsection{Auxiliary Results}
\begin{lemma}[{\cite[(6.6)]{borgs2012convergent}}]\label{lem:spectralrelations}
	Let $G$ be a weighted graph and $\lambda$ the spectrum interpreted as a point  measure. Let $C_k$ be the cycle of length $k$em. Then
	\[
		t(C_k, G)=\sum_{w\in \lambda}w^k.
	\]
\end{lemma}
\begin{lemma}[{Corollary of \cite[p. 200]{achieser2013theory}}]\label{lem:jackson}
		Let $f$ be a 1-{Lipschitz} function on $[-1,1]$. Then there is a polynomial $p$ of degree $v$ such that $\|f-p \|_\infty \le \frac{3}{\pi v}$. 
	\end{lemma}
	\begin{lemma}[{\cite[Lemma 4.1]{sherstov2012making}}]\label{lem:markov}
		Let $\sum_{i=0}^v a_i x^{i}$ be a polynomial on $[-1,1]$ bounded by $M$. Then
		\[
			\lvert a_i\rvert \le (4e)^v M.
		\]
	\end{lemma}
\subsection{Proof of Theorem \ref{thm:main2}}
\begin{proof}[Proof of Theorem 2]
		Consider any coupling $(\lambda, \lambda')$ of $\bar \lambda$ and $\bar \lambda'$. One has by the definition of the Wasserstein distance $\mathcal{W}^1_{\mathcal{W}^1}$ and Kantorovich duality
		\begin{equation}
			\mathcal{W}_{\W}^1(\bar \lambda', \bar\lambda) \le    \E\left[\W (\lambda, \lambda')\right] =  \E\left[\sup_{\operatorname{Lip} (f) \le 1}\int f(x) \dd (\lambda - \lambda')\right]\label{eq:doublewasserstein}
		\end{equation}
		Fix any $\omega \in \Omega$. By Lemma Lemma \ref{lem:jackson} one can approximate Lipschitz functions by polynomials of bounded degree,
		\[
			\sup_{\substack{f\colon [-1,1] \to \R\\\operatorname{Lip} (f) \le 1}}\int f(x) \dd (\lambda - \lambda')(\omega) \le \sup_{\substack{\operatorname{deg} (f) \le v\\ \lvert f \rvert \le 2 }}\int f(x) \dd (\lambda - \lambda')(\omega) + \frac{3}{\pi v}.
		\]
		Here, $\lvert f \rvert \le 2$ can be assumed as $f$ is defined on $[-1,1]$ and because of its $1$-Lipschitz continuity.
		
		Hence, by Lemma \ref{lem:markov} and the triangle inequality
		\begin{align*}
			\sup_{\substack{\operatorname{deg} (f) \le v\\ \lvert f \rvert \le 2 }}\int f(x) \dd (\lambda - \lambda') (\omega) &\le \sum_{i=1}^v 2(4e)^v \left\lvert \int x^k \dd (\lambda - \lambda')\right\rvert (\omega)\\
			&=  \sum_{i=1}^v 2(4e)^v \left\lvert \sum_{w \in \lambda} w^i - \sum_{w' \in \lambda'} w^i \right\rvert(\omega)
		\end{align*}
		Taking expectations, one gets
		\[
		\W_{\W} (\bar \lambda, \bar \lambda' ) \le \frac{3}{\pi v} + \sum_{i=1}^v 2(4e)^v \E \left [\left\lvert \sum_{w \in \lambda} w^i - \sum_{w' \in \lambda'} w^i \right\rvert\right]
		\]
		for any coupling $(\lambda, \lambda')$ of $\bar \lambda$ and $\bar \lambda'$. Now consider a coupling $(\lambda, \lambda')$ of $\bar \lambda$ and $\bar \lambda'$ such that $\bar t$, $\bar t'$ (which are functions of $\lambda$, $\lambda'$ by Lemma \ref{lem:spectralrelations}) are optimally coupled. Then by the definition of $\bar\lambda$, $\bar \lambda'$, $\bar t$ and $\bar t'$,
		\[
			\mathcal{W}^1 (\bar t_k, \bar t_k') = \E\left[ \left\lvert \sum_{w \in \ \lambda}w^k -  \sum_{w \in \bar \lambda'}w'^k\right\rvert\right]
		\]
		where $\bar t_i = \frac{1}{n} \sum_{j=1}^n\delta_{t(C_i, G_j)}$ and $\bar t_i' = \frac{1}{n} \sum_{j=1}^n\delta_{t(C_i, G_j)}$. Hence,
		\begin{align}
			\mathcal{W}_{\W}^1(\bar \lambda', \bar\lambda) &\le \sum_{i=1}^v 2(4e)^v \W (\bar t_i, \bar t'_i) +  \frac{3}{\pi v}.\label{eq:intermediatewassersteinbound}\\
			\nonumber &\le \frac{3}{\pi v} + v^2 2 (4e)^v \delta_\Box (W, W') + 18 v (4e)^v n^{-\frac{1}{3}}.
		\end{align}
The first equality follows by \eqref{eq:doublewasserstein} and the second with probability at least $1-2v\exp\left(\frac{kn^{-\frac{2}{3}}}{2v^2}\right)-2ve^{-.09cn^{\frac{2}{3}}}$ from Theorem \ref{thm:main1}. 
\end{proof}
\section{A Similar Bound for Degree Features}\label{sec:degreefeatures}
Let $G$ be a graph and $(d_i)$ be its degree sequence. Consider the point measure $d = \sum_{i} \delta_{d_i}$ of degrees. Denote by $\bar d$ resp. $\bar d'$ the empirical measure of degree point measures of $G_1, \dots, G_n$ resp. $G_1', \dots, G_n'$. 
\begin{proposition}
Theorem \ref{thm:main2} holds with the same guarantee with $\bar \lambda$, $\bar \lambda'$ replaced by $\bar d$, $\bar d'$. 
\end{proposition}
\begin{lemma}\label{lem:starlemma}
	Let $S_v$ be the star graph on $v$ nodes and $G$ be a weighted graph. Then
	\[
		t(S_v, G)=\sum_{w \in d}w^v
	\]
\end{lemma}
The proof of Proposition \ref{lem:spectralrelations} is along the same lines as the one of Theorem \ref{thm:main2}, but using Lemma \ref{lem:starlemma} instead of \ref{lem:spectralrelations}. 
\section{Heterogenous Sample Sizes}\label{sec:hetsampsizes}
Our bounds from Theorems \ref{thm:main1} and \ref{thm:main2} can also be formulated in a more general setting of heterogenous sizes of graphs. In the following, we give an extension in two dimensions. First, we allow for heterogenous numbers of observations $n$. Secondly, we allow for random sizes of graphs $k$. Here is the more general model in details: There is a measure $\nu \in \Pi (\N)$ such that $G_1, \dots, G_{n_1}$ are sampled iid as 
\begin{align}
	k &\sim \nu & G_i \sim \G(k,\mathcal{W}_{W, \mu});\label{eq:sampling}
\end{align}
sampling of $G_1', \dots, G_{n_2}'$ is analogously. Hence the samples $G_i$ are sampled from a mixture over the measures $\G(k,\mathcal{W}_{W', \mu'})$. We can define $\bar t$, $\bar t'$, $\bar \lambda$ and $\bar \lambda'$ using the same formulas as we did in the main text. Then the following result holds.
\begin{corollary} \label{cor:sampling}
	There is an absolute constant $c$ such that the following holds: Let $n_1, n_2 \in \N$ and $G_i, i=1, \dots, n_1$, $G_i', i=1, \dots, n_2$ sampled as in \eqref{eq:sampling}. Then with probability at least $1-\exp\left(\frac{kn_1^{-\frac{2}{3}}}{2e(F)^2}\right)-e^{-.09cn_1^{\frac{2}{3}}}-\exp\left(\frac{kn_2^{-\frac{2}{3}}}{2e(F)^2}\right)-e^{-.09cn_2^{\frac{2}{3}}}$, 
	\begin{equation}
	\delta_\Box(W, W') \ge e(F)^{-1} (\W(t, \bar t)- 5n_1^{-\frac{1}{3}}+5n_2^{-\frac{1}{3}}).
\end{equation}
\end{corollary}
\begin{corollary}
In the setting of Corollary \ref{cor:sampling} and with the same absolute constant, the following holds: Let $v \in \N$. With probability $1-v\exp\left(\frac{kn_1^{-\frac{2}{3}}}{2v^2}\right)-ve^{-.09cn_1^{\frac{2}{3}}}-v\exp\left(\frac{kn_2^{-\frac{2}{3}}}{2v^2}\right)-ve^{-.09cn_2^{\frac{2}{3}}}$,
	\[
		\delta_\Box (W, W')\ge v	^{-2}2^{-1} (4e)^{-v} \left(\mathcal{W}_{\W}^1 (\bar \lambda, \bar \lambda')- \frac{3}{\pi v} - 18v(4e)^v(n_1^{-\frac{1}{3}}+n_2^{-\frac{1}{3}})\right)
	\]
\end{corollary}
	The proofs are very similar to the ones in the main text. For the differences in $n_1$ and $n_2$, the concentration results Lemmas \ref{lem:expbound} and \ref{lem:wassersteinconc} will have to be applied separately with different values of $n$. For the random values $k$, we can choose a coupling that couples random graphs of similar sizes, leading to the expressions in the Corollaries.
\end{document}